\documentclass[journal]{IEEEtran}
\usepackage{amsmath,amsfonts}
\usepackage{algorithmic}
\usepackage{algorithm}
\usepackage{array}
\usepackage[caption=false,font=normalsize,labelfont=sf,textfont=sf]{subfig}
\usepackage{textcomp}
\usepackage{stfloats}
\usepackage{url}
\usepackage{verbatim}
\usepackage{graphicx}
\usepackage{cite}
\newcommand{\dynfabrics}{\textit{Dynamic Fabrics}}
\newcommand{\ours}{\textit{AIKIDO}}
\newcommand{\figref}[1]{Figure #1}
\newcommand{\secref}[1]{Section #1}
\newcommand{\tblref}[1]{Table #1}
\newcommand{\algref}[1]{Algorithm #1}
\usepackage{bbold}            
\usepackage{multirow}         
\usepackage{svg}              
\usepackage{orcidlink}        
\usepackage{diagbox}          
\usepackage{amssymb}          
\usepackage{amsthm}
\theoremstyle{definition}
\newtheorem{theorem}{Theorem}
\newtheorem{lemma}{Lemma}
\newtheorem{proposition}{Proposition}

\newtheorem{remark}{Remark}
\newtheorem{assumption}{Assumption}
\newtheorem{definition}{Definition}
\newcommand{\assuref}[1]{Assumption #1}
\newcommand{\propref}[1]{Proposition #1}

\newcommand{\lemmref}[1]{Lemma #1}
\newcommand{\remaref}[1]{Remark #1}

\begin{document}

\title{A Task-Space Receding Horizon Controller for Fast Collision Avoidance}

\author{Mattia Penzotti, Marco Controzzi
\thanks{The authors are with the Biorobotics Institute and the Department of Excellence in Robotics and AI, Scuola Superiore Sant'Anna, Pisa, Italy. \texttt{\{mattia.penzotti, marco.controzzi\}@santannapisa.it}}
}


\maketitle

\begin{abstract}
Real-time collision avoidance for robotic manipulators requires fast reactions to unexpected obstacle motion and enough lookahead to avoid becoming trapped by near-future constraints. Full model predictive control can provide this foresight, but its online cost may grow quickly with horizon length, model fidelity, and the number of active geometric constraints. Conversely, horizon-free reactive methods are computationally efficient but can be short-sighted in dynamic clutter. This paper presents a task-space receding-horizon controller that uses a short contact-consistent rollout to generate a terminal kinematic reference satisfying internal non-penetration constraints, then computes only the first input of a smooth minimum-acceleration transition toward that reference. Starting from a closed-loop inverse-kinematics regulation law, the rollout is performed with an iterative dynamics solver operating on inflated convex robot and obstacle geometries, so that robot--obstacle contacts, dynamic obstacle motion, and self-collision interactions can shape the terminal reference without requiring full constrained trajectory optimization. We analyze the contact-inactive closed loop and show local exponential task-space regulation under standard regularity assumptions. For contacts activated inside the rollout, we characterize the corresponding discrete updates and bound the effect of moving obstacles on regular operating sets. Simulations on a 40-degree-of-freedom (DOF) multi-chain system show that intermediate horizons balance anticipation, responsiveness, and computational cost. Hardware experiments on a 6-DOF platform demonstrate consistent sim-to-real behavior without accurate inertial parameter estimation, and comparisons against dynamic optimization fabrics and model predictive control (MPC) baselines show improved success rates in dynamic clutter while preserving solve times compatible with real-time execution in the tested regimes.

The open-source implementation can be found at \url{https://github.com/sssa-human-robot-interaction-lab/sim}
\end{abstract}


\section{Introduction}
Real-time collision avoidance for robotic manipulators requires a controller to react quickly to unexpected obstacle motion while also anticipating near-future interactions before clearances become critical. This requirement is particularly demanding in cluttered dynamic environments, where high-dimensional robot geometry, self-collision constraints, and moving obstacles must be handled within tight control-cycle budgets. Methods with explicit lookahead, such as model predictive control (MPC), provide a principled framework for reasoning about future constraints and stability under appropriate modeling and feasibility assumptions~\cite{dos2024set,palmieri2021motion}. However, their online computational cost can increase rapidly with model complexity, horizon length, and the number of active geometric constraints~\cite{haffemayer2025collision,santos2023nonlinear}. Horizon-free reactive planners, including local optimization, barrier-based filters, velocity dampers, and geometry-shaped vector fields, often scale more favorably but rely primarily on instantaneous geometry, which can make them short-sighted under fast obstacle motion, narrow passages, or dense clutter~\cite{ratliff2021generalized,van2022geometric,rakita2021collisionik,morton2025safe,spahn2023dynamic}.

The central idea of this paper is that useful anticipation does not necessarily require optimizing an entire constrained control sequence. Instead, a short-horizon forward rollout infers a contact-consistent near-future kinematic state that accounts for inflated-geometry non-penetration constraints. This terminal state allows predicted obstacle motion and self-collision interactions to influence the current command, while execution is reduced to a smooth transition from the measured state toward the predicted terminal configuration.

We instantiate this idea in \ours{}, a task-space receding-horizon controller, for manipulator collision avoidance. Starting from a closed-loop inverse-kinematics regulation law, \ours{} performs a short rollout in maximal coordinates using an iterative dynamics solver with non-penetration constraints between inflated convex robot and obstacle geometries. The rollout terminal state is mapped back to joint coordinates and used as the boundary condition of a minimum-acceleration optimal-control problem. Only the first acceleration input is applied, yielding a receding-horizon controller that combines local feedback, contact-consistent lookahead, and smooth executable joint commands. In this sense, \ours{} occupies an intermediate regime: it is not a full nonlinear MPC method, because it does not optimize a constrained state--input trajectory online, and it is not purely reactive, because contacts encountered during the rollout shape the terminal reference supplied to the controller.

Before detailing the contributions, we use the following abbreviations for the three controller layers: closed-loop inverse kinematics (CLIK), iterative-dynamics solver (IDS), and optimal-control (OC). The main contributions of this work are as follows:
\begin{enumerate}
    \item \textbf{A contact-consistent receding-horizon controller.}
    We introduce \ours{}, a task-space controller that combines closed-loop inverse kinematics, iterative-dynamics rollouts, and minimum-acceleration receding-horizon execution. The controller uses the contact solver to generate a short-horizon terminal reference shaped by dynamic obstacles and self-collision interactions, without solving a full constrained nonlinear MPC problem online.
    \item \textbf{A nominal stability and rollout-energy analysis.}
    We show that the contact-inactive CLIK--OC flow yields exponentially stable task-space regulation under standard local regularity assumptions. For contacts activated inside the IDS rollout, we characterize static-obstacle and self-collision updates as dissipative in the normal contact coordinates and bound dynamic-obstacle energy injection on regular operating sets. The analysis supports a practical robustness interpretation of the rollout mechanism without constituting a global safety certificate.
    \item \textbf{A broad empirical validation.}
    We evaluate \ours{} in high-dimensional simulation, hardware experiments, and head-to-head comparisons against representative reactive and predictive baselines\footnote{Multimedia available at \url{https://penzottimattia.github.io/aikido-frame-viewer/}}. The results show that \ours{} maintains real-time solve times while achieving strong collision-avoidance performance in dynamic clutter.
\end{enumerate}
The remainder of the paper is organized as follows. Section~II reviews predictive, reactive, geometric, and complementarity-based approaches to manipulator collision avoidance and positions \ours{} relative to each family. Section~III states the problem setting and scope. Section~IV introduces the contact-consistent rollout preliminaries. Section~V presents the \ours{} controller architecture, execution law, and analysis. Section~VI evaluates hyperparameter effects in high-dimensional simulation. Section~VII reports hardware results. Section~VIII compares \ours{} against representative predictive and reactive baselines. Section~IX discusses limitations and future work, and Section~X concludes the paper.
\section{Related Work}
\subsection{Predictive Collision Avoidance}
Predictive controllers address collision avoidance by optimizing future state and input trajectories subject to task and safety constraints. MPC is therefore a natural framework for robotic manipulation in cluttered environments, since it can reason over a receding horizon and encode obstacle-avoidance requirements directly in the optimization problem~\cite{dos2024set,santos2023nonlinear}. Recent work has extended this idea to whole-body and manipulator controllers that include non-penetration constraints, velocity-dependent safety margins, or dynamic-obstacle interactions inside the MPC loop~\cite{heins2023keep,haffemayer2025collision}. These methods explicitly account for future interactions, but their computational cost depends strongly on horizon length, model fidelity, and the number of active geometric constraints. In dense dynamic scenes, practical implementations often rely on simplified collision models, shortened horizons, or carefully tuned costs and constraints~\cite{santos2023nonlinear,heins2025robust}.

\subsection{Instantaneous Safety Filters}
A second family of methods enforces safety through instantaneous constraints, including velocity dampers and control barrier functions (CBFs). Velocity-damper formulations couple clearance with relative velocity and can be embedded as hard constraints inside control or optimization loops~\cite{haffemayer2025collision}. CBF-based controllers provide forward-invariance properties under appropriate assumptions and can be combined with operational-space control to obtain high-rate safety filters for manipulation~\cite{morton2025safe}. Recent approaches also combine CBF constraints with MPC, allowing tracking and safety objectives to be co-designed over a prediction horizon~\cite{liu2025flexible}. These mechanisms provide strong local safety structure, but can become conservative or myopic when obstacles move quickly, when many constraints activate simultaneously, or when the robot must move before the current state becomes unsafe.

\subsection{Geometric and Horizon-Free Reactive Planners}
Geometric fabrics and related horizon-free planners generate motion by shaping local vector fields or solving per-instant optimization problems that combine task tracking, obstacle avoidance, smoothness, and other objectives~\cite{ratliff2021generalized,van2022geometric}. These methods can achieve fast control cycles because they do not optimize over a future horizon. Extensions such as dynamic optimization fabrics incorporate obstacle motion and geometry-aware behavior with low computational overhead~\cite{spahn2023dynamic,spahn2023autotuning}. Other local optimization approaches, such as per-instant inverse-kinematics formulations, similarly exploit instantaneous geometric information to generate collision-avoiding motions~\cite{rakita2021collisionik}. However, because these methods are governed by local geometry, metric design, or weighted objective combinations, their performance may degrade when task attraction and obstacle repulsion become strongly conflicting~\cite{duhe2021contributions,wilde2024scalarizing}.

\subsection{Contact- and Complementarity-Based Motion Generation}
Contact and complementarity formulations provide a direct way to encode unilateral non-penetration constraints. Linear complementarity problem (LCP) formulations and iterative contact solvers are standard tools in rigid-body simulation, where they resolve contacts through impulse-momentum laws and non-penetration constraints~\cite{catto2005iterative,bender2014interactive,featherstone2014rigid,Carpentier-RSS-24}. Such formulations are appealing for robot collision avoidance because they model unilateral constraints directly, rather than approximating them through smooth potentials or distance penalties. Recent work has also explored complementarity-based reactive planning for whole-body collision avoidance, formulating collision-free motion generation through LCP-style constraints that explicitly account for robot geometry and rigid-body non-penetration~\cite{yao2025synthesis}.

\subsection{Paper positioning}
Relative to these families, \ours{} uses contact-consistent dynamics in a different role. Compared with nonlinear MPC, it does not optimize a constrained state--input trajectory online. Compared with instantaneous filters and geometric fabrics, it is not limited to the current constraint geometry. Compared with contact solvers used for simulation or instantaneous resolution, the IDS rollout generates a short-horizon terminal reference, while the OC layer separately computes the executable first acceleration. This use of an iterative contact solver as a terminal-reference generator is the distinguishing element of the proposed architecture, and preserves a degree of predictive behavior while keeping the online problem substantially simpler than full constrained trajectory optimization.
\section{Problem Setting and Scope}
\label{sec:problem}
\subsection{System and Inputs}
We consider a low-level, closed-loop manipulation controller operating at a fixed servo period $\Delta t$. At each cycle, the controller receives the measured joint state $(\mathbf{q},\dot{\mathbf{q}})$, a desired task-space target $\boldsymbol{x}^d$, convex or convex-decomposed inflated geometries for robot links and obstacles, and predicted obstacle velocities over a short horizon $T$. The controller outputs a joint-level acceleration command, which is integrated into the commanded joint velocity sent to the robot interface.

Let $g$ denote the forward-kinematics map from joint positions to task-space coordinates, so that $\boldsymbol{x}=g(\mathbf{q})$. We define the task-space regulation error as $\boldsymbol{\phi}=\boldsymbol{x}^d-g(\mathbf{q})$. Collision checking and rollout updates are performed on inflated robot and obstacle geometries, where inflation by the clearance value $\rho_0$ turns near-contact configurations into active non-penetration constraints inside the internal prediction model.

\subsection{Objective and Scope}
The objective is local task-space regulation with online collision avoidance in dynamic clutter. The controller should reduce the task-space error $\boldsymbol{\phi}$ while using the inflated geometries to discourage robot--obstacle and non-adjacent robot--robot penetrations through the internal rollout. The intended operating regime is low-level servo control, where decisions must be computed within tight control-cycle budgets and updated continuously from measured feedback.

We do not address global path planning, perception, obstacle-state estimation, or formal forward-invariance certification for the executed continuous trajectory. Non-penetration constraints are enforced inside the internal rollout model, whereas the robot executes the first input of the OC transition toward the rollout terminal state. The guarantees and empirical claims should therefore be interpreted as local control and rollout-consistency results, not as global planning or hard safety certificates.

\subsection{Design Requirements}
Within this scope, the controller is designed around three practical requirements. First, it should preserve local task-space regulation when no rollout contacts are active. Second, it should allow future obstacle motion to affect the current command over a short prediction horizon. It should also handle robot--obstacle and self-collision interactions through inflated-geometry non-penetration constraints, to account for uncertain/approximated geometries. Third, it should output smooth joint-level commands suitable for high-rate execution on physical robots. These requirements motivate the three-layer architecture developed in Section~\ref{sec:aikido}.
\section{Preliminaries}
\label{sec:preliminaries}
This section introduces the contact-consistent rollout machinery used by \ours{}. We relate general concepts of iterative dynamics~\cite{catto2005iterative,bender2014interactive,featherstone2014rigid} to the collision-avoidance setting considered here. Unlike standard formulations used in many iterative engines~\cite{Carpentier-RSS-24}, we express the rigid-body dynamics quantities in robot variables while retaining the notation of iterative contact dynamics~\cite{catto2005iterative}. The rollout is expressed in maximal coordinates, which allows rigid-body contact updates to be applied directly to link poses and velocities before mapping the terminal state back to joint space.

We consider an articulated serial chain comprising $n$ links. Given its state $\mathbf{q}(t)$ and $\dot{\mathbf{q}}(t)$ in reduced coordinates, we define $X(\mathbf{q}(t))$ and $V(\dot{\mathbf{q}}(t))$ as the spatial poses and velocities of the centers of mass, both represented as $6n$-by-$1$ maximal-coordinate vectors. The center-of-mass positions $X(t)$---for simplicity, we omit the dependency on the joint state---are computed via forward kinematics, while the velocities $V(t)$ are derived via Jacobian matrices.
The rollout model computes a contact-consistent state $X^+$ and $V^+$ at time $t+T$, which, within one IDS substep or over the horizon after repeated substeps, is expressed using semi-implicit Euler integration as:
\begin{equation}
    X^+ = X + V^+ T
    \label{eq:euler}
\end{equation}
where superscripts are used to drop explicit dependency on time.

\subsection{Design Choices and Assumptions}
Our design choices for the iterative dynamic solver are intended to enable contact-consistent lookahead predictions based solely on known geometries and constraints (i.e., robot joints and non-penetration constraints). In particular, by working with maximal coordinates, we can directly apply established rigid-body contact dynamics models~\cite{catto2005iterative,bender2014interactive,featherstone2014rigid,Carpentier-RSS-24} without the need to reformulate them for joint-space variables. This approach simplifies collision detection, distance computation, and the enforcement of contact constraints such as non-penetration and impulse-momentum. Moreover, we further streamline those models under additional assumptions on the collision avoidance problem.

\begin{assumption}[Obstacle updates]
\label{ass:obst_vels} 
In the proposed formulation, velocity updates for obstacles are never considered. This choice reflects two design criteria: (i) estimating obstacle states lies beyond the scope of a low-level controller in practical deployments, and (ii) the controller’s primary responsibility is to achieve collision avoidance rather than influence obstacle motion. Consequently, obstacle velocities are treated as fixed input parameters within the contact model, i.e., constant during the lookahead horizon, compelling the robot to adapt its dynamics to those of moving obstacles—never the reverse—by design.
\end{assumption}
    
\begin{assumption}[Absence of restitution]
The absence of restitution prevents robot links from bouncing upon contact with static obstacles, ensuring that they precisely maintain the safety distance defined by a \textit{clearance} value $\rho_0$.
\end{assumption}

\begin{assumption}[Inertial properties]
\label{ass:unit_mass} 
For the inertial properties of the robot links, we assume unit masses and identity inertia tensors. \assuref{\ref{ass:obst_vels}} is equivalent to treating all obstacles as having infinite mass, which removes the need for precise estimation of robot parameters. As in many classical closed-loop regulation schemes, we rely on direct measured feedback to compensate for possible tracking errors of low-level control inputs (e.g., joint velocities).
\end{assumption}

\subsection{Inflated-Geometry Contact Constraints}
\label{sec:contacts}

In the following, we review contact models that are particularly relevant to the present work. We consider obstacles to be static unless they have a non-zero center-of-mass velocity, in which case we refer to them as dynamic obstacles.  Collision pairs are restricted to interactions involving either two non-adjacent links of the articulated system (self-collisions) or a link and an obstacle.  First, we address the general case of resolving penetrations, which enables efficient handling of collisions with static obstacles. Next, we discuss the normal velocity constraint, which is fundamental for effectively avoiding collisions with dynamic obstacles. Finally, we cover the impulse-momentum law, which pertains to cases of self-collisions. We recall that, for both static and dynamic obstacles, center-of-mass velocities are not updated within the contact models. Throughout the paper, \textit{contact} denotes activation of a non-penetration constraint between inflated geometries; thus it does not necessarily imply physical contact between the robot and the obstacle.

\subsubsection{Static obstacle contacts and penetrations}
For potential collision pairs, we can identify the distance normal $\mathbf{n}$ through the Separating Axis Theorem (SAT) and SAT-based algorithms \cite{jimenez20013d}. The set $\mathcal{N} := \{\mathbf{n}_1,\ldots,\mathbf{n}_s\}$ collects all $i$ normals where the distance $d_i$ is less than or equal to $\rho_0$. As an $s$-by-$1$ vector, $C_\mathcal{N} := \{\left(\rho_0 - d_i\right),~\forall i \in \left[1,s\right]\}$ gathers the measured penetrations. Under our convention, normals are oriented outward of robot links.

An articulation is considered to be in collision whenever the set $\mathcal{N}$ is not empty. To handle this, a Baumgarte correction scheme \cite{baumgarte1972stabilization} updates the centers of mass velocities to produce proper separation effects as follows: 
\begin{equation}
    J_\mathcal{N} V^+ = -\beta C_\mathcal{N}
    \label{eq:baum}
\end{equation}
where $J_\mathcal{N}$ represents the $s$-by-$6n$ constraint Jacobian, which maps velocity updates at the contact point to velocity changes of the center of mass for any link involved in the collision. The left side of \eqref{eq:baum} can be interpreted as the time derivative of $C_\mathcal{N}$. Thus, the equation:
\begin{equation}
    \label{eq:baum_exp}
    \dot{C}_\mathcal{N}  + \beta C_{\mathcal{N}} = 0
\end{equation}
represents a decaying exponential for the position error, given $\beta > 0$. Additionally, note that \eqref{eq:baum} satisfies Newton's hypothesis for static obstacles and no restitution \cite{bender2014interactive}. 

\subsubsection{Dynamic obstacle contacts}

In the case of dynamic obstacles, it is crucial to consider the obstacle's closing velocity at the contact point. Thus, the $s$-by-$1$ vector $S_\mathcal{N}$ collects all of these terms, where the number of nonzero terms corresponds to the number of dynamic obstacles in a collision. To update the link velocities accordingly, we reformulate \eqref{eq:baum} as:
\begin{equation}
    J_\mathcal{N} V^+ = S_\mathcal{N} - \beta C_\mathcal{N}
    \label{eq:contact}
\end{equation}
By definition, this equation satisfies Newton's hypothesis with no restitution for collisions involving dynamic obstacles. 

\subsubsection{Self-collision contacts}

When a self-collision occurs, the velocities of both links involved must be updated according to the impulse-momentum laws. To achieve this, we compute the relative velocity at the contact point and add to $S_\mathcal{N}$ as an additional set of nonzero terms, equal to the number of self-collision pairs. Since each link of the robot is assumed to have the same mass and inertia tensor, we can reformulate \eqref{eq:contact} as follows:
\begin{equation}
    J_\mathcal{N} V^+ = P_\mathcal{N} \cdot J_\mathcal{N} V + \left(1 - \frac{ P_\mathcal{N}}{2}\right) \cdot S_\mathcal{N} - \beta C_\mathcal{N}
    \label{eq:impulsive}
\end{equation}
where $P_\mathcal{N}$ is an $s$-by-$1$ vector such that $p_i = 1$ if $n_i$ belongs to a self-collision pair, and $p_i = 0$ otherwise. Notably, \eqref{eq:impulsive} satisfies Newton's hypothesis with no restitution for self-collision between rigid bodies with identical inertial properties \cite{featherstone2014rigid}.

\subsection{Linear Complementarity Formulation}

The constrained equations of motion are expressed as:
\begin{equation}
    \frac{V^+ - V}{T} = M^{-1} J_\mathcal{N}^{\top} \lambda
    \label{eq:motion}
\end{equation}
where $\lambda$ represents an unknown $s$-by-$1$ vector of contact forces (or constraint multipliers) and $M(\mathbf{q})$ is the articulated body inertia tensor \cite{featherstone1999divide}. In lieu of \eqref{eq:impulsive}, we can rewrite this as:
\begin{multline}
    \frac{1}{T}\left[\left(1 - \frac{P_\mathcal{N}}{2}\right) \cdot S_\mathcal{N} - \beta C_\mathcal{N} + (P_\mathcal{N} - 1) \cdot  J_\mathcal{N} V \right] = \\ J_\mathcal{N} M^{-1} J_\mathcal{N}^{\top} \lambda
\label{eq:model}
\end{multline}
Here, the left-hand side is denoted by $\eta$, and $G=J_\mathcal{N} M^{-1} J_\mathcal{N}^{\top}$ denotes the Delassus matrix \cite{Carpentier-RSS-24}. Since constraint forces can only act in a repulsive manner (i.e. $\lambda \geq 0$), we can express this constraint as a linear complementarity problem (LCP):
\begin{align}
    w &= G \lambda - \eta & 0 = w &\bot \lambda \geq 0 
    \label{eq:lcp}
\end{align}
The constraint velocity vector $w$ represents the residual velocities following the application of the contact forces; here, $w \equiv 0$ indicates that the constrained equations of motion are satisfied as prescribed in \eqref{eq:model} \cite{catto2005iterative}.  

To solve \eqref{eq:lcp}, we can exploit the Projected Gauss-Seidel iterative method \cite{poulsen2010heuristic}. This allows us to determine $\lambda$ and compute the penetration-free state $X^+$ by combining \eqref{eq:motion} and \eqref{eq:euler} as follows:
\begin{equation}
    X^+ = X + V T + M^{-1} J_\mathcal{N}^{\top}T^2 \lambda
    \label{eq:predicted}
\end{equation}

\section{AIKIDO}
\label{sec:aikido}
\subsection{Controller Architecture} \label{sec:loop}

\begin{algorithm}[t]
\caption{One \ours{} control update}
\label{alg:approach}
\begin{algorithmic}[1]
\REQUIRE measured state $(\mathbf{q},\dot{\mathbf{q}})$, task target $\boldsymbol{x}^d$,
inflated geometries, obstacle velocities, horizon $T$, substeps $N$
\ENSURE commanded joint velocity $\dot{\mathbf{q}}_{\mathrm{cmd}}$

\STATE \textbf{// Closed-loop inverse kinematics}
\STATE $\boldsymbol{\phi}\gets \boldsymbol{x}^d-g(\mathbf{q})$
\STATE $\dot{\mathbf{q}}_d\gets \mathrm{CLIK}(\boldsymbol{\phi},\mathbf{q})$
\hfill $\triangleright$ Eq.~\eqref{eq:task}

\STATE \textbf{// Initialize contact-consistent rollout}
\STATE $(X,V)\gets \bigl(X(\mathbf{q}),V(\dot{\mathbf{q}}_d)\bigr)$
\hfill $\triangleright$ reduced $\rightarrow$ maximal
\STATE $T_s\gets T/N$

\STATE \textbf{// IDS rollout}
\FOR{$i=1,\ldots,N$}
    \STATE $(\eta,G)\gets \mathrm{evaluateModel}(X,V)$
    \hfill $\triangleright$ Eq.~\eqref{eq:model}
    \STATE $(X,V)\gets \mathrm{evaluateSolver}(\eta,G,T_s)$
    \hfill $\triangleright$ Eqs.~\eqref{eq:lcp}, \eqref{eq:predicted}
\ENDFOR

\STATE \textbf{// Terminal reference and receding-horizon action}
\STATE $\mathbf{q}_{X^+}\gets \mathbf{q}(X)$
\hfill $\triangleright$ maximal $\rightarrow$ reduced
\STATE $\ddot{\mathbf{q}}_0\gets
\mathrm{OC}(\mathbf{q},\dot{\mathbf{q}},\mathbf{q}_{X^+})$
\hfill $\triangleright$ Eq.~\eqref{eq:ocp}
\STATE $\dot{\mathbf{q}}_{\mathrm{cmd}}\gets
\dot{\mathbf{q}}+\ddot{\mathbf{q}}_0\Delta t$

\RETURN $\dot{\mathbf{q}}_{\mathrm{cmd}}$
\end{algorithmic}
\end{algorithm}
\ours{} consists of three modules executed at each servo cycle. The closed-loop inverse-kinematics (CLIK) layer generates a nominal joint velocity that reduces the task-space error. The iterative-dynamics solver (IDS) rolls this nominal motion forward under inflated-geometry contact constraints and returns a contact-consistent terminal joint configuration. The optimal-control (OC) layer then computes the first acceleration of a smooth minimum-acceleration transition from the measured state to that terminal configuration. The intermediate IDS states are used only internally; the executed command is produced by the OC layer from the rollout terminal state, as shown in \algref{\ref{alg:approach}}.
\subsection{Closed-Loop Inverse Kinematics}
The CLIK layer defines the nominal task-regulation field used to initialize the rollout. Given the measured configuration and task target, it outputs a joint velocity reference that reduces task error while using null-space motion to bias the robot toward a nominal posture.
Recall that $g$ denotes the forward-kinematics map from joint positions to task-space coordinates, so that $\boldsymbol{x} = g(\mathbf{q})$, and the regulation error $\boldsymbol{\phi} = \boldsymbol{x}^d - \boldsymbol{x}$. The joint-space velocity reference $\dot{\mathbf{q}}_d$ is then given by: 
\begin{equation} 
 \dot{\mathbf{q}}_d 
 = K J^\dagger \boldsymbol{\phi} 
 + k_0 B \left(\tilde{\mathbf{q}} - \mathbf{q} \right),
 \label{eq:task} 
\end{equation} 
where $J^\dagger$ denotes the pseudo-inverse of the Jacobian, computed using Singular Value Filtering (SVF) \cite{colome2012redundant}. 
$K$ is a positive definite diagonal matrix with entries equal to a scalar gain $k$, and $k_0 > 0$ is a posture gain. 
The null-space projection operator $B = \left( I - J^\dagger J \right)$ is used to resolve redundancy by attracting the configuration toward a nominal posture $\tilde{\mathbf{q}}$.

\subsection{Contact-Consistent Rollout} \label{sec:ids}
The rollout layer converts the nominal CLIK velocity into a contact-consistent terminal reference. Starting from the current configuration, it propagates the nominal motion over the horizon using the preliminary contact model of Section~\ref{sec:preliminaries}.
Given $X(\mathbf{q})$ and $V(\dot{\mathbf{q}}_d)$, we employ \eqref{eq:model}--\eqref{eq:predicted} with substepping \cite{macklin2019small} to derive a contact-consistent terminal state $X^+$ at time $t+T$. 

The use of substepping, combined with SAT-based collision detection as described in \secref{\ref{sec:contacts}}, allows for online updates of collision shapes before collecting the active contact set $\mathcal{N}$. This includes progressively augmenting the solver with swept-volume approximations of dynamic obstacles, thereby making evasive motions more conservative with respect to predicted obstacle motion \cite{Michaux-RSS-24}. 

Because the IDS propagates contact impulses consistently with the reduced-coordinate articulation model \cite{featherstone1999divide}, the corresponding joint configuration is directly available at the end of the rollout. We denote this terminal joint state by $\mathbf{q}_{X^+}$.

\subsection{Minimum-Acceleration Receding-Horizon Execution}
The rollout trajectory itself is not executed. Instead, \ours{} uses only the terminal configuration produced by the IDS as a short-horizon reference. The intermediate sequence may exhibit discontinuities in reduced coordinates due to substepping and contact corrections, so rollout non-penetration should be interpreted as an internal prediction constraint rather than as a certificate for the executed continuous trajectory. 

Instead, a minimum-acceleration optimal control problem is solved to obtain a smooth, hardware-compatible transition toward the predicted terminal state. Only the first control input is applied, yielding a receding-horizon implementation. 

Moreover, repeatedly tracking the terminal rollout \cite{ghazaei2015online} state produces smooth motion toward near-future configurations that encode anticipatory collision-avoidance behavior when $\lambda > 0$. 

Formally, the OC drives the system from the current state $\{\mathbf{q}, \dot{\mathbf{q}}\}$ to the predicted terminal state $\{\mathbf{q}_{X^+}, 0\}$ by minimizing joint accelerations $\ddot{\mathbf{q}} \equiv \mathbf{u}$: 
\begin{equation} 
\label{eq:ocp} 
\begin{aligned} 
 \min_{\mathbf{x}(t),\,\mathbf{u}(t)} \quad 
 & \int_{0}^{T} \mathbf{u}(t)^\top \mathbf{u}(t)\,dt \\
 \text{s.t.} \quad 
 & \mathbf{x}(t) = \begin{bmatrix} \mathbf{p}(t) \\ \mathbf{v}(t) \end{bmatrix}, \quad 
 \dot{\mathbf{x}}(t) = \begin{bmatrix} \mathbf{v}(t) \\ \mathbf{u}(t) \end{bmatrix}, \quad 
 t \in [0, T] 
\end{aligned} 
\end{equation} 
with boundary conditions 
$\mathbf{x}(0)=\bigl[\mathbf{q},\,\dot{\mathbf{q}}\bigr]^\top$ and 
$\mathbf{x}(T)=\bigl[\mathbf{q}_{X^+},\,0\bigr]^\top$. 

The zero terminal velocity should be interpreted as a regularizing boundary condition for the local transition, not as a requirement that the robot stops at the predicted state; since the problem is solved at every control cycle, motion is continuously updated in receding-horizon fashion. 

We ignore $\dot{\mathbf{q}}_{V^+}$ from the IDS, as velocities may contain spurious terms induced by the Baumgarte correction scheme in \eqref{eq:baum}. 

The solution corresponds to a cubic polynomial trajectory $\mathbf{p}(t) = \mathbf{a}_0 + \mathbf{a}_1 t + \mathbf{a}_2 t^2 + \mathbf{a}_3 t^3$. 
The resulting optimal initial acceleration is given by:
\begin{equation}
\ddot{\mathbf{q}}_0 
= 
\frac{6}{T^2} \left( \mathbf{q}_{X^+} - \mathbf{q} \right)
- 
\frac{4}{T} \dot{\mathbf{q}},
\end{equation}
which is applied as the control input at the current timestep.

\subsection{Nominal Stability and Rollout-Contact Energy Bounds}
\label{sec:stability_energy}
The preceding subsections define \ours{} as a composition of nominal regulation, contact-consistent terminal-reference generation, and smooth receding-horizon execution. We analyze this composition in two parts: first, the contact-inactive CLIK--OC flow; second, the bounded effect of contact activations inside the discrete IDS rollout.

Let recall $\boldsymbol{\phi}$ denotes the task-space regulation error, $\mathbf{q},\dot {\mathbf{q}}$ the joint variables, $X,V$ the maximal-coordinate pose/velocity used by the rollout, $\mathcal{N}$ the active rollout-contact set, $C_\mathcal{N}\!\ge 0$ the normal penetration-with-respect-to-clearance vector, $J_\mathcal{N}$ the constraint Jacobian, $S_\mathcal{N}$ the obstacle normal speeds, $P_\mathcal{N}\in\{0,1\}^s$ the self-collision indicator, $\lambda\!\ge 0$ the normal contact multipliers, and $G=J_\mathcal{N} M(\mathbf{q})^{-1}J_\mathcal{N}^\top$ the Delassus matrix. We assume $M(\mathbf{q})\succ 0$ \cite{featherstone2014rigid,featherstone1999divide} and use the no-restitution contact model \cite{bender2014interactive,featherstone2014rigid,Carpentier-RSS-24}.

\begin{assumption}[Regular operating set]
\label{ass:controller_contact}
The following analysis applies on a compact regular operating set where: (i) CLIK and OC are as in \eqref{eq:task}, \eqref{eq:ocp} with $K=kI$, $k>0$, horizon $T>0$, and the task Jacobian is sufficiently well conditioned after SVF regularization so that $JJ^\dagger\approx I$ and higher-order kinematic terms are negligible locally ($\dot{J}\dot{\mathbf{q}} \approx 0$); (ii) the IDS contact problem is well posed, with $G\succ 0$ or regularized so that the LCP \eqref{eq:lcp} admits bounded multipliers satisfying $w=G\lambda-\eta\ge 0$, $\lambda\ge 0$, $w^\top\lambda=0$; (iii) obstacle normal speeds are bounded, $\|S_\mathcal{N}\|\le S_{\max}$; and (iv) commanded velocities, clearances, and contact dwell activity are bounded by the servo implementation and the inflated-geometry model.
\end{assumption}

\begin{remark}[Scope of the analysis]
\assuref{\ref{ass:controller_contact}} is a local regularity condition, not a global safety certificate. In particular, near kinematic singularities or nearly dependent contact normals, the controller may require regularization, reduced IDS iterations, or a conservative fallback action, as discussed in \secref{\ref{sec:discuss}}. Contacts in this subsection are rollout-contact activations between inflated geometries and do not necessarily correspond to physical impacts of the robot.
\end{remark}

\begin{definition}[Contact-inactive rollout regime]
The rollout is contact-inactive if the IDS detects no active inflated-geometry contact constraints over the current prediction horizon, i.e., $\mathcal{N}=\emptyset$.
\end{definition}

\begin{proposition}[Contact-inactive regulation]
\label{prop:dyn}
In the contact-inactive rollout regime, the IDS terminal state reduces to nominal propagation over the horizon, inducing the joint-space reference $\mathbf{q}_{X^+}=\mathbf{q}+\dot{\mathbf{q}}_dT$. Neglecting null-space contributions, the OC initial acceleration is
\[
\ddot{\mathbf{q}}_0=\frac{6}{T}KJ^\dagger\boldsymbol{\phi}-\frac{4}{T}\dot{\mathbf{q}}.
\]
Under \assuref{\ref{ass:controller_contact}}, the task-space error locally satisfies
\[
\ddot{\boldsymbol{\phi}}+\frac{4}{T}\dot{\boldsymbol{\phi}}+\frac{6k}{T}\boldsymbol{\phi}=0 .
\]
Consequently, the contact-inactive CLIK--OC flow is exponentially regulating for $T>0$ and $k>0$.
\end{proposition}

\begin{proof}
When $\mathcal{N}=\emptyset$, \eqref{eq:predicted} contains no contact impulse and the terminal joint reference is $\mathbf{q}_{X^+}=\mathbf{q}+\dot{\mathbf{q}}_dT$. Substituting this into the OC law gives the stated acceleration. With $\dot{\mathbf{q}}_d=KJ^\dagger\boldsymbol{\phi}$ and $K=kI$, the local kinematics give $\dot{\boldsymbol{\phi}}\approx -J\dot{\mathbf{q}}$ and $\ddot{\boldsymbol{\phi}}\approx -J\ddot{\mathbf{q}}$. Using $JJ^\dagger\approx I$ yields the stated second-order linear system. Its characteristic polynomial $s^2+\frac{4}{T}s+\frac{6k}{T}$ is Hurwitz for $T>0$ and $k>0$, proving exponential regulation of the nominal flow.
\end{proof}

\begin{remark}[No-overshoot gain bound]
\label{rem:kmax}
The contact-inactive response is non-oscillatory when the discriminant of $s^2+\frac{4}{T}s+\frac{6k}{T}$ is nonnegative, i.e., when $k\le \frac{2}{3T}=:k_{\max}$. This gives a simple tuning guideline for selecting the CLIK gain relative to the rollout horizon.
\end{remark}

\begin{definition}[Nominal Lyapunov and rollout energy]
\label{def:energy_lyap}
For the nominal error dynamics, define
$V_\phi(\boldsymbol{\phi},\dot{\boldsymbol{\phi}}):=\frac{1}{2}\|\dot{\boldsymbol{\phi}}\|^2+\alpha\|\boldsymbol{\phi}\|^2$ with $\alpha>0$.
For the IDS rollout, define the maximal-coordinate kinetic energy $E_{\mathrm{kin}}:=\frac{1}{2}V^\top M(\mathbf{q})V$.
\end{definition}

\begin{proposition}[Nominal damping identity]
\label{prop:exp_flow}
For the contact-inactive dynamics of \propref{\ref{prop:dyn}}, choosing $\alpha=\frac{3k}{T}$ yields
\[
\dot V_\phi=-\frac{4}{T}\|\dot{\boldsymbol{\phi}}\|^2\le 0 .
\]
This identity provides an energy interpretation of the damping in the exponentially stable nominal flow.
\end{proposition}

\begin{proof}
Substituting the dynamics of \propref{\ref{prop:dyn}} into $\dot V_\phi=\dot{\boldsymbol{\phi}}^\top\ddot{\boldsymbol{\phi}}+2\alpha\boldsymbol{\phi}^\top\dot{\boldsymbol{\phi}}$ and setting $\alpha=3k/T$ cancels the cross term, yielding the stated expression. Exponential stability follows from the Hurwitz property established in \propref{\ref{prop:dyn}}, rather than from strict negativity of this particular Lyapunov derivative.
\end{proof}

\begin{lemma}[Static rollout contacts]
\label{lem:baumgarte_dissipation}
For static obstacles ($S_\mathcal{N}=0$, $P_\mathcal{N}=0$), the IDS imposes $J_\mathcal{N}V^+=-\beta C_\mathcal{N}$ as in \eqref{eq:baum}. Therefore, the post-update normal contact power satisfies
\[
\mathcal{P}_c^+=(J_\mathcal{N}V^+)^\top\lambda=-\beta C_\mathcal{N}^\top\lambda\le 0 .
\]
Thus the static rollout-contact update is dissipative in the normal contact coordinates.
\end{lemma}

\begin{proof}
At active inflated-geometry contacts, $C_\mathcal{N}\ge0$ and $\lambda\ge0$ by the LCP complementarity conditions. The imposed post-update normal velocity is separating with respect to the clearance error, so the normal contact power is non-positive. Under the standard no-restitution contact assumptions used by the IDS, this corresponds to removal of normal closing motion rather than injection of normal contact energy.
\end{proof}

\begin{lemma}[Self-collision rollout contacts]
\label{lem:self_collision}
For self-collisions ($P_\mathcal{N}=1$) with zero restitution and identical inertias as in \eqref{eq:impulsive}, the normal component of relative velocity is not increased by the update. Consequently, the self-collision update is non-increasing in the normal relative kinetic component.
\end{lemma}

\begin{proof}
Equation~\eqref{eq:impulsive} encodes Newton's hypothesis with $e=0$ for the normal component under identical inertial properties \cite{featherstone2014rigid}. The impulse acts along $J_\mathcal{N}^\top$ and removes normal closing motion; tangential components are not modeled by the frictionless normal LCP. Hence the normal relative kinetic component cannot increase.
\end{proof}

\begin{lemma}[Moving rollout contacts: bounded injection]
\label{lem:bounded_injection}
For moving obstacles, the IDS imposes $J_\mathcal{N}V^+=S_\mathcal{N}-\beta C_\mathcal{N}$ as in \eqref{eq:contact}. The post-update normal contact power is
\[
\mathcal{P}_c^+=-\beta C_\mathcal{N}^\top\lambda+S_\mathcal{N}^\top\lambda .
\]
On the regular operating set of \assuref{\ref{ass:controller_contact}}, there exist finite constants $c_C,c_S>0$ such that
\[
|\mathcal{P}_c^+|\le c_C\|C_\mathcal{N}\|+c_S\|S_\mathcal{N}\| .
\]
Thus dynamic obstacles can inject normal contact energy into the internal rollout, but the injection is bounded by clearance-error and obstacle-normal-speed terms on the regular set.
\end{lemma}

\begin{proof}
By Cauchy--Schwarz,
$|\mathcal{P}_c^+|\le(\beta\|C_\mathcal{N}\|+\|S_\mathcal{N}\|)\|\lambda\|$.
On the regular operating set, $G$ is uniformly nonsingular or regularized, and the right-hand side $\eta$ of the LCP is bounded because commanded velocities, clearance errors, and obstacle normal speeds are bounded. Hence the LCP solution has bounded multipliers, $\|\lambda\|\le \lambda_{\max}$, with $\lambda_{\max}$ depending on the compact operating set, contact geometry, and solver regularization. The stated bound follows with $c_C=\beta\lambda_{\max}$ and $c_S=\lambda_{\max}$.
\end{proof}

\begin{theorem}[Contact-inactive regulation and bounded rollout-contact perturbations]
\label{thm:rollout_bounds}
Under \assuref{\ref{ass:controller_contact}}, the contact-inactive CLIK--OC flow is exponentially regulating. When contacts are activated inside the IDS rollout, static-obstacle and self-collision updates are dissipative or non-increasing in the normal contact coordinates, while moving-obstacle updates introduce bounded normal-energy perturbations proportional to clearance-error and obstacle-normal-speed terms on the regular operating set. Consequently, under bounded contact activity and solver convergence, the receding-horizon controller admits a practical robustness interpretation with respect to rollout-contact perturbations; the bound degrades with obstacle normal speed, clearance activity, and contact dwell time.
\end{theorem}

\begin{table*}[t]
\centering
\caption{Success counts for simulation experiments (80 trials per condition)}
\label{tab:benchmarking}
\begin{tabular}{l|llll|llll|llll|llll|}
Obstacles (\#) & \multicolumn{4}{c|}{10}                                  & \multicolumn{4}{c|}{20}                          & \multicolumn{4}{c|}{30}                          & \multicolumn{4}{c|}{40}                 \\ \hline \diagbox{$\rho_0~(m)$}{$T~(s)$}
               & 0.2         & 0.4          & 0.6         & \textit{Tot.} & 0.2 & 0.4         & 0.6          & \textit{Tot.} & 0.2         & 0.4          & 0.6 & \textit{Tot.} & 0.2 & 0.4         & 0.6 & \textit{Tot.} \\ \hline
0.02           & \textbf{79} & \textbf{79}  & \textbf{79} & \textbf{237}  & 72  & 73          & \textbf{74}  & \textbf{219}  & \textbf{62} & \textbf{62}  & 58  & \textbf{182}  & 45  & \textbf{49} & 42  & \textbf{136}  \\
0.04           & 18          & \textbf{73}  & 71          & 162           & 56  & 66 & \textbf{70}           & 192           & 50          & \textbf{54}  & 49  & 153           & 31  & \textbf{38} & 8   & 77            \\ \hline
\textit{Tot.}  & 97          & \textbf{152} & 150         &               & 128 & 139         & \textbf{144} &               & 112         & \textbf{116} & 107 &               & 76  & \textbf{87} & 50  &              
\end{tabular}
\end{table*}

\begin{table*}[t]
\centering
\caption{Mean (std) solve time in milliseconds for simulation experiments}
\label{tab:benchmarking_solve}
\begin{tabular}{l|crr|crr|crr|crr|}
Obstacles (\#) & \multicolumn{3}{c|}{10}                                                       & \multicolumn{3}{c|}{20}                                                       & \multicolumn{3}{c|}{30}                                                       & \multicolumn{3}{c|}{40}                                                       \\ \hline\diagbox{$\rho_0~(m)$}{$T~(s)$}
               & 0.2                      & \multicolumn{1}{c}{0.4} & \multicolumn{1}{c|}{0.6} & 0.2                      & \multicolumn{1}{c}{0.4} & \multicolumn{1}{c|}{0.6} & 0.2                      & \multicolumn{1}{c}{0.4} & \multicolumn{1}{c|}{0.6} & 0.2                      & \multicolumn{1}{c}{0.4} & \multicolumn{1}{c|}{0.6} \\ \hline
0.02           & \multicolumn{1}{r}{4(1)} & 6(2)                    & 7(3)                     & \multicolumn{1}{r}{4(1)} & 6(1)                    & 7(2)                     & \multicolumn{1}{r}{4(1)} & 6(2)                    & 8(3)                     & \multicolumn{1}{r}{3(1)} & 6(2)                    & 8(3)                     \\
0.04           & \multicolumn{1}{r}{8(2)} & 12(4)                   & 16(6)                    & \multicolumn{1}{r}{8(3)} & 12(5)                   & 16(8)                    & \multicolumn{1}{r}{8(4)} & 13(7)                   & 18(11)                   & \multicolumn{1}{r}{8(4)} & 14(8)                   & 23(18)                  
\end{tabular}
\end{table*}

\begin{proof}
The contact-inactive statement follows from \propref{\ref{prop:dyn}}. The nominal damping identity of \propref{\ref{prop:exp_flow}} shows how CLIK--OC dissipates task-space error energy during contact-inactive flow. When the IDS activates contacts, \lemmref{\ref{lem:baumgarte_dissipation}} and \lemmref{\ref{lem:self_collision}} show that static and self-contact updates do not inject normal contact energy, while \lemmref{\ref{lem:bounded_injection}} bounds the contribution of moving obstacles. Since only the terminal rollout state is passed to the OC, these contact effects enter the executed controller as bounded perturbations of the near-future reference. Under bounded contact activity over finite horizons, repeated receding-horizon updates therefore preserve practical boundedness, with the neighborhood size increasing with the accumulated rollout-contact perturbations. This result should be interpreted as a regular-regime robustness statement, not as a global forward-invariance or hard safety guarantee for the continuous executed trajectory.
\end{proof}

\section{Simulation Experiments}

In this section, we conduct extensive experiments to benchmark the computational efficiency of \ours{} while also examining the influence of its hyperparameters in collision environments of increasing complexity. Although there is a fixed relationship to guide the choice of flow gain $k$, the design choice of a suitable clearance value $\rho_0$ and the selection of an appropriate horizon $T$ have so far not been investigated. This set of experiments provides practical guidelines for tuning \ours{} as scene complexity, such as the number of collision elements, increases.

We evaluate the controller in a high-dimensional setting composed of six distinct serial chains (4$\times$7-DOF, 2$\times$6-DOF), for a total of 40 DOF. We use a control rate of 50 Hz with $k = 0.2 k_{\max}$ (\remaref{\ref{rem:kmax}}). Furthermore, we do not rely on simplified collision primitives to approximate link geometries, but we employ more precise convex hull representations.

Finally, we focus on dynamic-obstacle scenarios with increasing obstacle density, as shown in \tblref{\ref{tab:benchmarking}}. We select two levels for the clearance value and three lengths for the horizon, while keeping the iterative solver step fixed $T_s = 0.04~\text{s}$, as this substepping \cite{macklin2019small} resolution is sufficient to provide numerical stability \cite{erez2015simulation}. For each condition, defined by a tuple ($\rho_0$,$T$,$n$) with $n$ denoting the number of obstacles, we conduct 80 valid trials, as follows.

To generate obstacles, we first define the workspace $\mathbb{W} = [0,1] \times [-1,1] \times [0,1]$, which includes both the robot geometries and the target positions for each serial chain. Within this workspace, restricted regions are introduced as vertically oriented cylinders with radii large enough to prevent dynamic obstacles from colliding with the robot bases and other fixed links—components that, due to the specific kinematic design of each chain, cannot be freely positioned in three-dimensional (3D) space. Next, we sample the source and sink points for dynamic obstacles $\overrightarrow{p} \sim U\left(\mathbb{W}\right)$, and project them onto randomly chosen distinct faces of the workspace. If the line connecting these two projected points does not intersect any restricted region, we assign the obstacle a velocity magnitude $v \sim U \left([0.1,0.3]\right)$ in the corresponding direction. Otherwise, we repeat the process by selecting new candidate points.

To assess whether trajectories generated by \ours{} were collision-free, we post-processed each recorded configuration by evaluating the closest-point distances between the robots and the obstacles with a third-party collision library \cite{coumans2021}. This sampled post-hoc check defines empirical success, but not a continuous-time safety proof. All experiments, including those in the following sections, are conducted on an AMD Ryzen 9 3900X 12-Core CPU.

\subsection{Success Rate}

Across obstacle densities, an intermediate prediction horizon generally provides the most reliable performance. Aggregating across clearances in \tblref{\ref{tab:benchmarking}}, $T=0.4\,\mathrm{s}$ achieves the best totals in three of the four densities—152/160 with 10 obstacles, 116/160 with 30 obstacles and 87/160 with 40 obstacles—while $T=0.6\,\mathrm{s}$ holds a modest edge at 20 obstacles (144/160 versus 139/160 for $T=0.4\,\mathrm{s}$). The clearance parameter $\rho_{0}$ interacts strongly with the horizon length. Increasing $\rho_{0}$ from $0.02\,\mathrm{m}$ to $0.04\,\mathrm{m}$ contracts feasible corridors and, in dense scenes, drives conservative behavior that impedes progress; this is most pronounced at high density and long horizons, where at 40 obstacles and $T=0.6\,\mathrm{s}$ the success count falls from 42/80 to 8/80. At the other extreme—low density with a very short horizon—the larger clearance still reduces reliability; with 10 obstacles and $T=0.2\,\mathrm{s}$, successes drop from 79/80 to 18/80. 

These trends are consistent with the proposed receding-horizon interpretation. For very short horizons, the forward rollout provides limited anticipatory information about imminent interactions with moving obstacles, and the closed loop is therefore dominated by task-space regulation; in this regime the admissible non-oscillatory regulation gain is comparatively higher $\bigl(k_{\max}=\tfrac{2}{3T}\bigr)$, further biasing the response toward rapid error correction rather than early evasive action. For longer horizons, the degradation is instead attributable to the structure of the optimal controller: the minimum-acceleration boundary-value solution distributes the required adjustment over the full interval $[0,T]$, so that the first receding-horizon input is attenuated as $T$ increases, effectively diluting reactive avoidance corrections at execution time.

\subsection{Solve Time}

Per-cycle solve times scale predictably with obstacle density, horizon length, and clearance (\tblref{\ref{tab:benchmarking_solve}}). With $\rho_{0}=0.02\,\mathrm{m}$, costs remain low and grow gently with $T$; at 30 obstacles, mean times are approximately 4, 6, and 8 ms for $T=0.2, 0.4,$ and $0.6\,\mathrm{s}$, respectively, with modest trial‑to‑trial variability. Increasing the clearance to $\rho_{0}=0.04\,\mathrm{m}$ raises both the mean and variability, especially at long horizons and high density; at 40 obstacles, mean times increase to roughly 8, 14, and 23 ms for $T=0.2, 0.4,$ and $0.6\,\mathrm{s}$. Relative to a 50 Hz control budget ($\approx20\,\mathrm{ms}$ per cycle), nearly all conditions at $\rho_{0}=0.02\,\mathrm{m}$ remain comfortably real‑time, even as obstacle count grows. The only regime that risks exceeding the budget is the most cluttered configuration at the largest clearance and longest horizon—40 obstacles, $\rho_{0}=0.04\,\mathrm{m}$, $T=0.6\,\mathrm{s}$—where the mean reaches $\sim23\,\mathrm{ms}$ with high variability. Taken together with the success‑rate results, $T \approx 0.4\,\mathrm{s}$ with $\rho_{0} = 0.02\, \mathrm{m}$ appears to provide a favorable balance of anticipatory reasoning, reactive correction, and computational cost in the tested scenarios (\tblref{\ref{tab:benchmarking}} and \ref{tab:benchmarking_solve}).

\begin{table}[t]
\centering
\caption{Hardware validation: sim-to-real dynamic obstacle avoidance (30 trials, 5 obstacles)}
\label{tab:simtoreal}
\resizebox{1\columnwidth}{!}{
\begin{tabular}{cc|cc}
                                                                                                               & \begin{tabular}[c]{@{}c@{}}Wilcoxon (Z)\\ or t-test (T)\end{tabular} & \textsc{sim} & \textsc{real} \\ \hline
\multirow{2}{*}{\begin{tabular}[c]{@{}c@{}}Action cost\\ (rad/s\textsuperscript3)\end{tabular}} & \multirow{2}{*}{T,p}                                                 & 4.65 ± 2.21                   & 4.87 ± 2.31                    \\
                                                                                                               &                                                                      & \multicolumn{2}{c}{0.23, 0.82}                                 \\ \hline
\multirow{2}{*}{\begin{tabular}[c]{@{}c@{}}Min. clearance\\ (mm)\end{tabular}}                                 & \multirow{2}{*}{Z,p}                                                 & 21.91 ± 1.68                  & 21.81 ± 2.10                   \\
                                                                                                               &                                                                      & \multicolumn{2}{c}{0.26, 0.79}                                 \\ \hline
\multirow{2}{*}{\begin{tabular}[c]{@{}c@{}}Time to goal\\ (s)\end{tabular}}                                    & \multirow{2}{*}{Z,p}                                                 & 10.33 ± 1.54                  & 10.24 ± 1.68                   \\
                                                                                                               &                                                                      & \multicolumn{2}{c}{0.04, 0.97}                                 \\ \hline
\multirow{2}{*}{\begin{tabular}[c]{@{}c@{}}Solver time\\ (ms)\end{tabular}}                                    & \multirow{2}{*}{Z,p}                                                 & 0.49 ± 0.04                   & 0.94 ± 0.15                    \\
                                                                                                               &                                                                      & \multicolumn{2}{c}{6.49, \textless 0.001***}    \\ \hline
Success rate                                                                                                   & -                                                                    & 0.97                          & 0.97                          
\end{tabular}}
\end{table}

\section{Hardware Validation}

All hardware experiments were conducted on a Universal Robots UR10e platform. 
Controller hyperparameters were selected according to the simulation study, with horizon $T=0.4$ s, substep resolution $T_s = 0.04$ s, and clearance value $\rho_0 = 0.02$ m. 
The experiments were designed to evaluate sim-to-real consistency, robustness to aggressive obstacle motion, and suitability for high-rate execution on physical hardware. 
Across all experiments, collision checking relied on inflated convex robot geometries, while obstacle states were updated either from simulated trajectories or from motion-capture estimates. To match the update rate of the UR10e, we set the control frequency to 500 Hz and the task-space error gain $k = 0.3 k_{\max}$ (\remaref{\ref{rem:kmax}}).

\begin{figure*}[t]
    \centering
    \includegraphics[width=\linewidth]{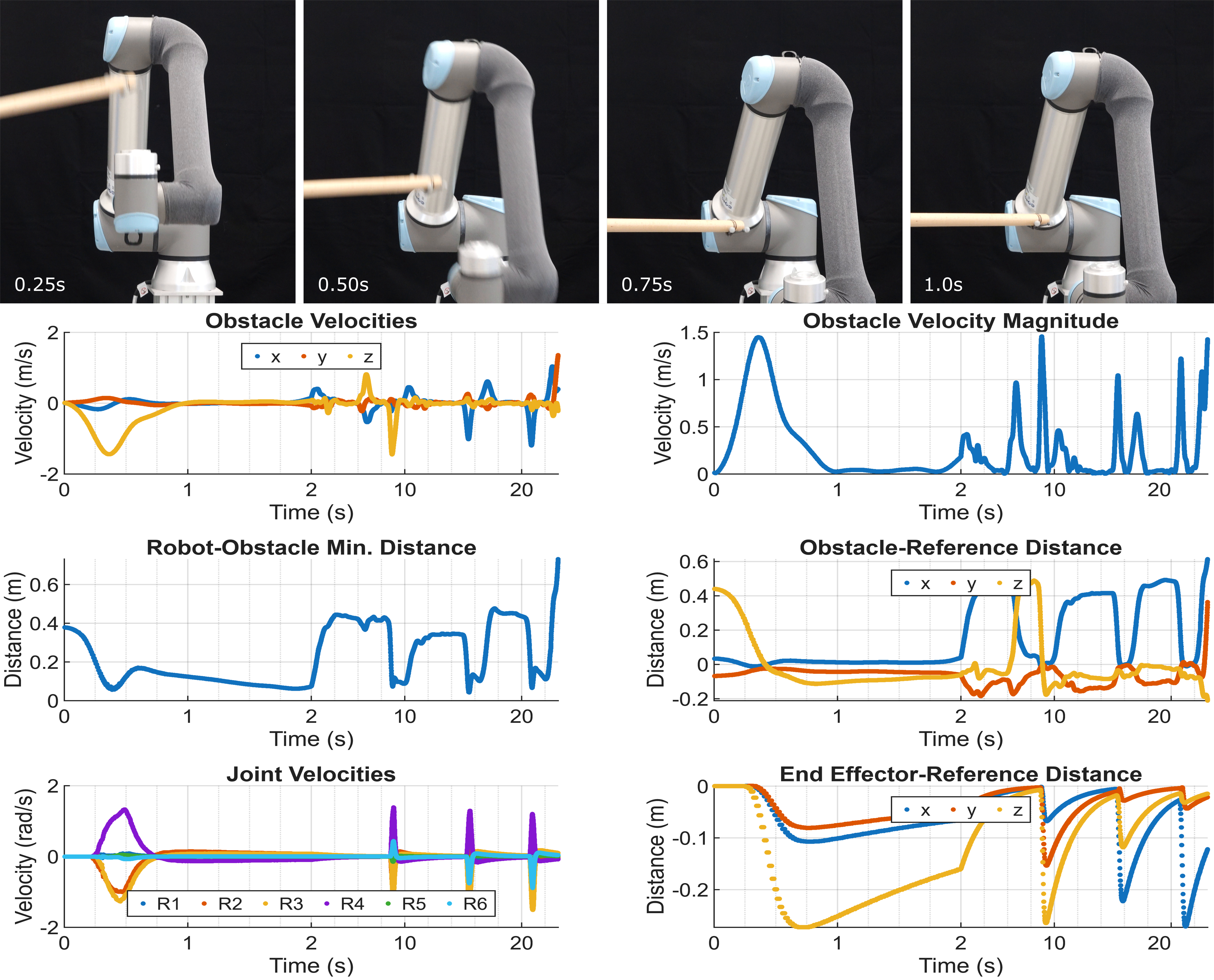}
    \caption{Fast obstacle avoidance with a tracked real moving obstacle, with a few representative sequences. All vectors in the top-left, middle-right, and bottom-right panels are expressed in the robot base frame. The obstacle quickly approaches the robot end-effector from different directions (top left). As the obstacle center approaches the target \textit{reference} end-effector position, i.e., the origin (middle-right), the robot moves away from the target to ensure clearance (bottom-right). As a result, the minimum clearance stays positive (middle left) even with high peaks in approach velocity magnitude (top right), while joint velocities exhibit smooth profiles (bottom left).}
    \label{fig:fast}
\end{figure*}

\subsection{Sim-to-Real Dynamic Obstacle Avoidance}

In this set of experiments, we evaluated the controller using perception-free (simulated) obstacles to isolate the effects of modeling assumptions and control-loop execution, independent of sensing uncertainties. Specifically, we aim to verify that closing the control loop at fast cycle times maintains effectiveness on a real manipulator, even without estimating inertial parameters for the IDS model, as per \assuref{\ref{ass:unit_mass}}. Thus, in addition to measuring the \textit{solver time}, we also consider \textit{time to goal}, \textit{action cost} (defined as the time integral of the norm of joint accelerations, consistent with the minimization objective in \eqref{eq:ocp}), and \textit{minimum clearance} (the smallest distance between the robot and obstacles during successful trials).

In this task, the robot must reach a target position $\overrightarrow{p} \sim U([0.35,0.65]^3)$ from a fixed home configuration, $\overrightarrow{p}_0 = \{0.5, -0.5, 0.5\}$ (all measures in meters and aligned with respect to the robot base frame). The home configuration is fixed in joint angles across all trials. We sample target positions until a geometric rapidly exploring random tree (RRT) planner finds a feasible motion plan \cite{karaman2011anytime}, within a 0.01 m tolerance.

The goal is to reach the target while avoiding collisions with 5 moving obstacles $\overrightarrow{o} \sim U\left([0.2,0.8] \times [0,1] \times [0.2,0.8]\right)$. The obstacles are spheres with a radius of 0.1 m. At least one obstacle is initialized at the end-effector target position, with velocity $\overrightarrow{v} = \{0,-0.1,0\}$, so that no trivial solution exists for the reaching task. The remaining obstacles move with constant linear velocity $\overrightarrow{v} \sim U \left(0 \times [-0.1,-0.4] \times [-0.2,0.2]\right)$ so that the obstacles generally move toward the end-effector.

We compared the performance obtained in simulation with that measured on real hardware, using a Universal Robots UR10e as the test platform. This evaluation uses 30 trials in simulation and 30 trials on hardware.

Statistical results are reported in \tblref{\ref{tab:simtoreal}}. In the case of a given metric, if both samples are normally distributed, we report mean values and standard deviations, with t-statistics and p-values from two-sample t-tests. Otherwise, we report median values and interquartile ranges, with z-statistics and p-values from the Wilcoxon rank-sum test. We consider statistical significance for p-values below the 1\% threshold.

We did not observe significant statistical differences between simulation results and real-world performance of \ours{}, except for the solver time metric, which may be affected by the additional overhead introduced by encapsulating the algorithm within the \textit{ROS-control} framework \cite{chitta2017ros_control}. Importantly, the minimum clearance between the manipulator links and the obstacle shapes is generally consistent with the specified clearance value.

\subsection{Fast Obstacle Avoidance with Real Moving Obstacles}

We further demonstrate real-world fast obstacle avoidance under perception-driven obstacle tracking. An off-the-shelf motion capture suite tracks online the position of four markers rigidly connected to a stick. The computed centroid is interpreted as the obstacle center position of a sphere with a radius of 0.05 m, and its velocity is obtained through Kalman filtering. With respect to the previous experiment, we use the estimates to update the obstacle pose and velocity between control ticks. The estimate is fed at approximately 300 Hz while the controller runs at 500 Hz.

As shown in \figref{\ref{fig:fast}}, the estimated obstacle velocity is significantly higher than in the previous setting, reaching magnitude peaks (top row, right column) of approximately 1.5 m/s while approaching the robot wrists from different angles (top row, left column). The closest-point distance between the sphere and robot convex hulls stays positive (middle left); no physical contact occurred in these experiments, while joint velocities remain smooth even during aggressive avoidance maneuvers (bottom left).

The remaining graphs (middle and bottom rows, right column) show, respectively, the obstacle centroid and end-effector distance from the current robot target position, which is fixed in these experiments. In other words, the robot stands still, moves as needed to ensure clearance with the approaching obstacle, and reaches back the given target once the obstacle is removed.

\section{Head-to-Head Comparison}

In simulation, we performed a controlled head-to-head comparison with a state-of-the-art reactive planner \cite{spahn2023dynamic} and an MPC baseline \cite{heins2025robust}, each of them selected for their proven ability to handle dynamic obstacles with ad-hoc formulations. In both cases, we select the authors' suggested hyperparameters, which are expert-tuned values for \dynfabrics{} \cite{spahn2023autotuning}, and original settings for the MPC baseline with dynamic obstacle avoidance \cite{heins2023keep}.

Concerning \ours{}, we select hyperparameters as in the previous hardware experiments but set $k = 0.5k_{\max}$ (\remaref{\ref{rem:kmax}}). The clearance value $\rho_0 = 0.02 ~\text{m}$ is also used for enforcement of non-penetration constraints in the MPC baseline. A hard timeout of 20 seconds is also considered.

For the head-to-head comparison with the authors' open-source implementations \cite{spahn2023dynamic,heins2025robust}, all methods, including \ours{}, use sphere primitives to approximate robot-link collision geometry. The obstacle-generation procedure, obstacle trajectories, timeout, robot model, success criterion, and post-hoc collision checking protocol are shared across methods; only method-specific hyperparameters follow the respective authors' recommended settings. Collisions detected by the post-hoc checker are counted as collision failures, while trials that do not reach the goal within 20 s are counted as timeouts. We compared performance on a simulated Franka Emika Panda arm at 100 Hz. Except for the platform and home joint configuration, the setup procedure for each trial was identical to that used in the UR10e hardware experiments. Results are shown in \tblref{\ref{tab:simulation}}.

The \dynfabrics{} baseline achieves relatively strong success rates, coupled with fast cycle times. Nevertheless, efficiency—measured as total joint-space path length—remains well below the level achieved through look-ahead search for optimality, as restored in the MPC baseline. In fact, the MPC baseline performs markedly better, though at the expense of a computational cost that is nearly two orders of magnitude higher in terms of solve time. This observation aligns with prior findings in \cite{spahn2023dynamic}. Regarding success rates, our experiments show that \dynfabrics{} scales slightly worse as obstacle density increases, with failures primarily attributable to collisions, as also observed in \cite{spahn2023dynamic}. \ours{} exhibits strong scalability across the tested range of conditions, maintaining efficient solver times while surpassing both baselines in terms of success rates. In terms of path efficiency, \ours{} lies between the reactive and predictive baselines, reflecting the inherent trade-offs in our approach.

\begin{table}[t]
\centering
\caption{Head-to-head comparison results}
\label{tab:simulation}
\resizebox{\columnwidth}{!}{
\begin{tabular}{rc|ccc|cc}
\multicolumn{1}{c}{}                                                                                     & \rotatebox{90}{Obst. (\#)} & \rotatebox{90}{Success}     & \rotatebox{90}{Collision} & \rotatebox{90}{Timeout} & \begin{tabular}[c]{@{}c@{}}Solve \\ time (ms)\end{tabular} & \begin{tabular}[c]{@{}c@{}}Travelled \\ distance (rad)\end{tabular} \\ \hline
\multicolumn{1}{r|}{\multirow{3}{*}{\textit{\begin{tabular}[c]{@{}r@{}}Dynamic\\ Fabrics\end{tabular}}}} & 1          & \textbf{80} & -         & -       & 0.36 ± 0.01                                                & 9.08 ± 1.10                                                         \\
\multicolumn{1}{r|}{}                                                                                    & 3          & 68          & 11        & 1       & 0.54 ± 0.02                                                & 11.55 ± 4.93                                                        \\
\multicolumn{1}{r|}{}                                                                                    & 5          & 56          & 23        & 1       & 0.69 ± 0.02                                                & 14.06 ± 9.19                                                        \\ \hline
\multicolumn{1}{r|}{\multirow{3}{*}{MPC}}                                                                & 1          & 75          & -         & 5       & 19.63 ± 1.06                                               & 6.73 ± 2.76                                                         \\
\multicolumn{1}{r|}{}                                                                                    & 3          & 70          & 7         & 3       & 37.02 ± 1.09                                               & 9.00 ± 2.38                                                         \\
\multicolumn{1}{r|}{}                                                                                    & 5          & 63          & 14        & 3       & 68.03 ± 1.72                                               & 9.48 ± 2.02                                                         \\ \hline
\multicolumn{1}{l|}{\multirow{3}{*}{\ours{}}}                                                    & 1          & \textbf{80} & -         & -       & 0.20 ± 0.01                                                & 8.02 ± 1.34                                                         \\
\multicolumn{1}{l|}{}                                                                                    & 3          & \textbf{80} & -         & -       & 0.21 ± 0.01                                                & 9.06 ± 2.77                                                         \\
\multicolumn{1}{l|}{}                                                                                    & 5          & \textbf{79} & 1         & -       & 0.22 ± 0.01                                                & 9.79 ± 4.11                                                        
\end{tabular}
}
\end{table}

\section{Discussion}
\label{sec:discuss}
While \ours{} achieves fast collision-aware regulation in the tested scenarios, the current implementation does not enforce explicit per-cycle time budgets or hard real-time guarantees. Under extreme combinations of dense clutter, large clearances, and long horizons, solver latency can exceed the control period. A practical extension is a deadline-aware execution model that caps IDS substeps and LCP iterations, degrades gracefully when budgets are tight, and triggers a constant-time fallback, such as velocity damping or braking, upon overrun.

The stability and contact-energy analysis also depends on regularity of both the task Jacobian and the contact problem. Near kinematic singularities, SVF limits amplification but weakens the approximation $JJ^\dagger \approx I$, which may reduce convergence rates. Similarly, multiple nearly dependent active contact normals can make the Delassus matrix ill-conditioned and degrade the LCP solution. In implementation, this condition can be monitored through $\kappa(G)$; excessive ill-conditioning can trigger regularization, reduced IDS iterations, or a conservative braking/velocity-damping policy.

Finally, the experiments indicate that the horizon is not a monotonic performance knob. It trades anticipation through the IDS rollout against responsiveness through the OC layer: short horizons provide limited lookahead, whereas long horizons can dilute the first minimum-acceleration action. This suggests that adaptive-horizon variants may be preferable to a fixed horizon, especially in scenes where obstacle speed, clearance activity, and available computational budget vary over time.
\section{Conclusions}
This paper presented \ours{}, a task-space receding-horizon controller that uses contact-consistent short-horizon rollouts to provide anticipatory collision avoidance without solving a full constrained nonlinear MPC problem online. The controller combines nominal task regulation, contact-consistent terminal-reference generation, and minimum-acceleration receding-horizon execution.

The analysis showed local exponential task-space regulation in contact-inactive regimes and bounded effects of rollout-contact perturbations under regularity assumptions. Empirically, simulations on a 40-DOF multi-chain system showed that intermediate horizons provide the best balance between anticipation, responsiveness, and computational cost. Hardware experiments on a UR10e demonstrated consistent sim-to-real behavior without accurate inertial parameter estimation, and head-to-head comparisons against \dynfabrics{} and an MPC baseline showed higher success rates in dynamic clutter while maintaining solve times close to reactive methods.

Overall, the results suggest that contact-consistent terminal rollouts provide a practical middle ground between purely reactive collision avoidance and full predictive optimization. Future work will focus on deadline-aware execution, adaptive substepping and horizon selection, robust handling of ill-conditioned contact configurations, perception-driven obstacle prediction, and validation with real moving obstacles and multi-arm interaction.
\bibliographystyle{IEEEtran}
\bibliography{include/references}



 





\end{document}